\documentclass[letterpaper]{article}

\usepackage[utf8]{inputenc}
\usepackage{amsmath,amscd,amssymb,amsthm}
\usepackage{fullpage}
\usepackage{natbib}
\usepackage{hyperref}

\usepackage{microtype}
\usepackage{graphicx}
\usepackage{subfigure}
\usepackage{booktabs} % for professional tables

\usepackage{amscd}
\usepackage{verbatim}
\usepackage{thmtools}
\usepackage{thm-restate}

\declaretheorem[name=Theorem]{Theorem}

\declaretheorem[name=Lemma, numberlike=Theorem]{Lemma}

\usepackage{algorithm}
\usepackage{algorithmic}

\title{Adaptive Online Learning with Varying Norms}

\author{
  \textbf{Ashok Cutkosky}\\Google Research\\\texttt{ashok@cutkosky.com}
}

\usepackage{times}
\newcommand{\R}{\mathbb{R}}
\newcommand{\Z}{\mathbb{Z}}

\newcommand{\argmin}{\mathop{\text{argmin}}}

\newcommand{\w}{\mathring{w}}
\newcommand{\x}{\mathring{x}}
\newcommand{\ol}{\mathcal{A}}
\newcommand{\trace}{\text{tr}}

\newcommand{\eff}{\text{eff}}

\newcommand{\vv}{\mathring{v}}
\newcommand{\wealth}{\text{Wealth}}

\begin{document}

\maketitle

\begin{abstract}
    Given any increasing sequence of norms $\|\cdot\|_0,\dots,\|\cdot\|_{T-1}$, we provide an online convex optimization algorithm that outputs points $w_t$ in some domain $W$ in response to convex losses $\ell_t:W\to \R$ that guarantees regret $R_T(u)=\sum_{t=1}^T \ell_t(w_t)-\ell_t(u)\le \tilde O\left(\|u\|_{T-1}\sqrt{\sum_{t=1}^T \|g_t\|_{t-1,\star}^2}\right)$ where $g_t$ is a subgradient of $\ell_t$ at $w_t$. Our method does not require tuning to the value of $u$ and allows for arbitrary convex $W$. We apply this result to obtain new ``full-matrix''-style regret bounds. Along the way, we provide a new examination of the full-matrix AdaGrad algorithm, suggesting a better learning rate value that improves significantly upon prior analysis. We use our new techniques to tune AdaGrad on-the-fly, realizing our improved bound in a concrete algorithm.
\end{abstract}

% \begin{keywords}%
%   Online Learning, Adaptive Algorithms, Stochastic Optimization
% \end{keywords}

\section{Introduction}
This paper provides new algorithms for online learning, which is a popular problem formulation for modeling streaming and stochastic optimization \citep{zinkevich2003online, cesa2006prediction, shalev07online}. Online learning is a game of $T$ rounds between an algorithm and the environment. In each round, the algorithm first chooses a point $w_t$ in some domain $W$, after which the environment presents the learner with a loss function $\ell_t:W\to \R$. Performance is measured by the \emph{regret}, which is a function of some benchmark point $\w$: $R_T(\w)=\sum_{t=1}^T \ell_t(w_t)-\ell_t(\w)$.
% One may use the online learning game to understand the process of iteratively training a machine learning model: $w_t$ is the $t$th iterate of the model parameters, and $\ell_t$ is the loss on the $t$th batch of training examples so that $R_T(\w)$ is the total loss experienced over the training algorithm relative to the total loss experienced by the benchmark or optimal model parameters $\w$.

In order to make the problem tractable, we will assume that each $\ell_t$ is convex and $W$ is a convex domain, which is often called \emph{online convex optimization}. Now, if we let $g_t$ be an arbitrary subgradient of $\ell_t$ at $w_t$, we have:
\begin{align*}
    R_T(\w)\le \sum_{t=1}^T \langle g_t, w_t - w\rangle
\end{align*}
Because of this fact, for the rest of this paper we consider exclusively the case of linear losses and take $\sum_{t=1}^T \langle g_t, w_t-\w\rangle$ as the definition of $R_T(\w)$. Well-known lower bounds \citep{abernethy2008optimal} tell us that even if the environment is restricted to $\|g_t\|_{\mathbf{2}}\le 1$ and $\|\w\|_{\mathbf{2}}\le 1$, no algorithm can guarantee regret better than $O(\sqrt{T})$ in all scenarios, and this bound is in fact obtained by online gradient descent \citep{zinkevich2003online}. In order to go beyond this minimax result, there is a large body of work on designing \emph{adaptive} algorithms \citep{auer2002adaptive, duchi10adagrad,mcmahan2010adaptive, mcmahan2012no,foster2015adaptive,orabona2014simultaneous, orabona2016coin,foster2018online,jun2019parameter,   kempka2019adaptive, van2019user}. A common goal of adaptive algorithms is to obtain a regret bound like:
\begin{align}
    R_T(\w)\le \|\w\|\sqrt{\sum_{t=1}^T\|g_t\|_\star^2}\label{eqn:constantnorm}
\end{align}
where $\|\cdot\|$ is some norm and $\|\cdot\|_\star$ is the dual norm. This type of bound is appealing: in the worst-case we never do worse than the minimax optimal rate, but in many cases we can do much better. For example, if $\|\w\|$ is small (intuitively, the benchmark point is ``simple''), or if the $\|g_t\|_\star$ values are small (intuitively, the losses are ``simple''), then we obtain low regret. 
%From another viewpoint, if we have some prior ``guess'' for what the optimal parameter $\w$ is, we can shift coordinates so that this guess is the origin, and then algorithm will suffer very low regret if the guess is correct, but not do too badly if the guess is incorrect. Further, it has been shown that the dependence on $\|g_t\|_\star^2$ implies asymptotically lower regret in many cases in which the losses are smooth \citep{srebro2010smoothness, levy2017online, cutkosky2019anytime}. 
The challenge in obtaining these kinds of bounds lies in the fact that the values that appear in the regret guarantee are unknown to the algorithm and so intuitively the algorithm must somehow learn about them on-the-fly.

In this paper, we provide a general technique for achieving adaptive bounds. Our primary result is an algorithm that takes a sequence of increasing norms $\|\cdot\|_0,\dots,\|\cdot\|_{T-1}$ and obtains regret
\begin{align}
R_T(\w)&\le \tilde O\left(\|\w\|_{T-1}^2\sqrt{\sum_{t=1}^T \|g_t\|_{t-1,\star}^2}\right)\label{eqn:varynorm}
\end{align}
The norms $\|\cdot\|_t$ may be generated on-the-fly (e.g. $\|\cdot\|_t$ can depend on $g_t$). Further, our algorithm can incorporate arbitrary convex domains $W$. Prior adaptive algorithms have typically required specific forms of $W$, such as being an entire vector space or having bounded diameter, and have often focused on a single norm.

As a corollary of this result, we obtain new ``full-matrix'' regret bounds. The first of these is:
\begin{align}
    R_T(\w)&\le \tilde O\left(\sqrt{r \sum_{t=1}^T \langle g_t, \w\rangle^2}\right)\label{eqn:fullmatrix}
\end{align}
where here $r$ is the rank of the subspace spanned by the $g_t$. Such a bound may be desirable because it allows the algorithm to in some sense ``ignore'' irrelevant directions in the $g_t$ by projecting them all along $\w$. This adaptivity comes at a price: all full-matrix algorithms to date (including ours) have substantially slower iterations than algorithms that obtain bounds like (\ref{eqn:constantnorm}) because they involve maintaining a $d\times d$ matrix, and so require at least $O(d^2)$ time per round in contrast to $O(d)$ time, where $d$ is the dimension of $W$, hence the name ``full-matrix''. Nevertheless, one can hope that the regret improves enough and makes up for the increased computational burden. We provide the first algorithm to achieve (\ref{eqn:fullmatrix}) for general convex domains $W$ rather than entire vector spaces.

Next, we provide a new analysis of the regret of the full-matrix AdaGrad algorithm \citep{duchi10adagrad}. Prior analysis of full-matrix AdaGrad yields a regret bound that is never better than using $\|\cdot\|=\|\cdot\|_{\mathbf{2}}$ in (\ref{eqn:constantnorm})\footnote{
Note that the prior bound for the \emph{diagonal} AdaGrad algorithm is different and can indeed provide gains over (\ref{eqn:constantnorm}).}. Nevertheless, full-matrix AdaGrad is empirically successful despite requiring slow matrix manipulations, suggesting that something is missing from the analysis. We posit that the missing ingredient is a suboptimal tuning of the learning rate, and show that with oracle tuning one can obtain the regret bound:
\begin{align}
    R_T(\w)&\le \tilde O\left[\sqrt{\left\langle \w, \sqrt{\sum_{t=1}^T g_tg_t^\top}\ \w\right\rangle\trace\sqrt{\sum_{t=1}^T g_tg_t^\top}}\right]\label{eqn:adagradbound}
\end{align}
We provide an interpretation of this bound suggesting that it allows for small regret when $\sum_{t=1}^T g_tg_t^\top$ is \emph{approximately} low-rank. Moreover, we can automatically achieve this oracle tuning as a simple corollary of our bound (\ref{eqn:varynorm}). Intriguingly, the three regret bounds (\ref{eqn:constantnorm}), (\ref{eqn:fullmatrix}), and (\ref{eqn:adagradbound}) are all incomparable - there are sequences of $g_t$ such that any one of them might be significantly better than the others.
%In the unconstrained setting, we could exploit this by applying the technique of combining regret bounds from \cite{cutkosky2019combining} to obtain an algorithm whose regret is the minimum of these possibilities.

Finally, we move beyond pure online linear optimization to consider linear supervised learning. This is a variant of online convex optimization for which in each round the algorithm is provided with a feature vector $f_t$ before it must decide on the output $w_t$. The loss $\ell_t$ is constrained to be of the form $\ell_t(w) = c_t(\langle f_t, w\rangle)$ for some convex function $c_t$. This describes learning with linear models, such as in logistic regression. A goal in this setting is to be \emph{scale-invariant}: the values $\langle f_t, w_t\rangle$ should be unchanged if the features are rescaled by some unknown factor, as explored by \cite{ross2013normalized,luo2016efficient, kotlowski2019scale, kempka2019adaptive}. Intuitively, scale-invariant algorithms are robust to using the ``wrong units'' to measure the features. Our techniques provide a new scale-invariant algorithm that improves logarithmic factors over prior analyses.

This paper is organized as follows: in Section \ref{sec:background}, we lay out our setting and introduce some background from the literature. In Section \ref{sec:learning}, we describe our primary technique and show how to achieve the bound (\ref{eqn:varynorm}). In Sections \ref{sec:fullmatrix} and \ref{sec:adagrad}, we show how to use our approach to achieve bounds (\ref{eqn:fullmatrix}) and (\ref{eqn:adagradbound}), and in Section \ref{sec:scale} we provide our results for scale-invariant algorithms.

\section{Preliminaries}\label{sec:background}

\subsection{Notation and Setup}
Throughout this paper we will make use of a variety of seminorms $\|\cdot\|$. We use $\|\cdot\|_0,\dots,\|\cdot\|_{T-1}$ to indicate an arbitrary sequence of $T$ potentially different seminorms. In order to avoid confusion between the $L_p$ norm and the $p$th element of a sequence of seminorms, we denote the $L_p$ norm using a bold font: $\|\cdot\|_{\mathbf{p}}$. When $W\subset \R^d$, we will also make use of the norm specified by a symmetric positive semi-definite matrix $M$ defined by $\|x\|_M=x^\top Mx$. We will use the notation $G_t=\sum_{i=1}^t g_tg_t^\top$ as a shorthand for the sum of the outer product of the loss vectors $g_t$. Finally, by abuse of notation we will write the dual of a seminorm as $\|x\|_{\star}=\sup_{\|y\|\le 1}\langle y, x\rangle$. Note that $\|x\|_{\star}$ may be infinity for some values of $x$ if $\|\cdot\|$ is a seminorm rather than a true norm.

We restrict our attention to those seminorms such that the function $\frac{1}{2}\|\cdot\|^2$ is $\sigma$-strongly-convex with respect to the same seminorm $\|\cdot\|$ for some $\sigma$. A function $f:W\to \R$ is $\sigma$-strongly-convex if for all $x$ and $y$ and $g\in \partial f$ we have $f(y)\ge f(x)+\langle g, y-x\rangle + \frac{\sigma}{2}\|x-y\|^2$. We have mildly relaxed the definition of strong-convexity to allow $\|\cdot\|$ to be a seminorm rather than a norm. All of the properties of strong-convexity we need in our analyses still hold under this definition. 
% For any symmetric positive-definite matrix $M$, the norm $\|x\|_M^2=\langle x, Mx\rangle$, satisfies this property with $\sigma=1$, as do the $L_p$ norms for $p\in(1,2]$ with $\sigma=p-1$. 
%Technically, this property tells us that the Banach space defined by the norm $\|\cdot\|$ is 2-smooth. %In fact, we can perform our same analysis techniques for $q$-smooth Banach spaces for $q\ne 2$ with essentially no change, but for simplicity we will focus our exposition on the 2-smooth case and present the more general analysis in the separate Section \ref{sec:banach}.

We will assume $W$ is a convex set for which it is possible to compute the projection operation $\Pi(x)=\argmin_{w\in W}\|w-x\|$ for any seminorm $\|\cdot\|$ we are interested in. We will also usually require $\|g_t\|_\star\le 1$ for the seminorms we consider. We recall for convenience here that $\|g\|_{M,\star}^2=\langle g, M^{-1} g\rangle$ for $g$ in the range of $M$ and infinity otherwise. The kernel and range of a symmetric matrix are orthogonal, so the use of inverse notation here results in a well-defined function.

Finally, in order to ease exposition we have suppressed many constants and occasionally a logarithmic factor in our main presentation. For completeness, we provide full characterizations of all our results including constant factors in the Appendix along with any proofs not in the main text.

\subsection{Follow-the-Regularized-Leader}

Follow-the-Regularized-Leader (FTRL) \citep{shalev07online} is one of the most successful abstractions for designing online convex optimization algorithms (see \cite{mcmahan2014survey} for a detailed survey). FTRL algorithms produces outputs $w_1,\dots,w_T$ through the use of \emph{regularizer} functions $\psi_0,\dots,\psi_{T-1}$. Specifically, $w_{t+1}$ is given by:
\begin{align*}
    w_{t+1} = \argmin_{w\in W} \psi_t(w) + \sum_{i=1}^{t} \langle g_i, w\rangle
\end{align*}
The following result from \cite{mcmahan2014survey} characterizes the regret of FTRL:
\begin{Theorem}[Adapted from \cite{mcmahan2014survey} Theorem 1]\label{thm:ftrl}
Suppose each $\psi_t$ is $\sigma_t$-strongly-convex with respect to a seminorm $\|\cdot\|_t$ for some $\sigma_t$, and $\psi_{t+1}(w)\ge \psi_t(w)$ for all $t$ and all $w\in W$. Further suppose $\inf_{w\in W}\psi_0(w)=0$. Then the regret of FTRL is bounded by:
\begin{align*}
    R_T(\w)\le \psi_{T-1}(\w) + \frac{1}{2}\sum_{t=1}^T \frac{\|g_t\|_{t-1,\star}^2}{\sigma_{t-1}}
\end{align*}
where recall we define $\|g\|_{\star}=\sup_{\|x\|\le 1} \langle g,x\rangle$ for any seminorm $\|\cdot\|$.
\end{Theorem}
The original presentation is stated for the case that $\|\cdot\|$ is a true norm, but it is relatively straightforward to check that nothing changes when we move to the seminorm case, where recall that we defined a ``dual seminorm'' analogously to the dual norm as $\|g\|_\star = \sup_{\|x\|\le 1}\langle x, g\rangle$.

The FTRL algorithm template has been used to great effect in the design of adaptive algorithms through clever choices of regularizer functions $\psi_t$. In particular, many prior works \citep{duchi10adagrad,mcmahan2010adaptive,orabona2016scale} achieve bounds similar to:
\begin{align*}
    \frac{1}{\eta}\|\w\|^2_{\mathbf{2}}\sqrt{\sum_{t=1}^T \|g_t\|_{\mathbf{2}}^2} + \eta \sqrt{\sum_{t=1}^T \|g_t\|_{\mathbf{2}}^2}
\end{align*}
for some fixed learning-rate scaling $\eta$ chosen by the user. Note that with the optimal tuning $\eta=\|\w\|_{\mathbf{2}}$, this bound recovers (\ref{eqn:constantnorm}) for $\|\cdot\|=\|\cdot\|_{\mathbf{2}}$.  Unfortunately, this value of $\eta$ is unknown a priori (and maybe even a posteriori) because we do not know what $\|\w\|_{\mathbf{2}}$ is.

\subsection{Parameter-Free Algorithms}
In an effort to fix the need to tune $\eta$ in FTRL algorithms, there has been a push for ``parameter-free'' algorithms that can adapt to unknown values of $\w$ \citep{mcmahan2012no, orabona2013dimension,   orabona2016coin, foster2017parameter, cutkosky2017online, foster2018online,  cutkosky2018black, kempka2019adaptive}. These algorithms make use of a known bound on the norm of $g_t$ in order to achieve adaptivity to $\|\w\|$. We will make use of the following recent bound (which is optimal up to constants and quantities inside logarithms):
\begin{Theorem}[Adapted from \cite{cutkosky2019matrix} Theorem 2]\label{thm:1dolo}
    For any user-specified values $\epsilon>0$ and $0\le Z \le 1$, there exists an online convex optimization algorithm with domain $W=\R$ that runs in time $O(1)$ per update such that if $|g_t|\le 1$ for all $t$, the regret is bounded by:
    \begin{align}
        R_T(\w)=\sum_{t=1}^Tg_t(w_t-\w)&\le  O\left[\epsilon+|\w|\max\left(\sqrt{\frac{1+\sum_{t=1}^T g_t^2}{Z}\log\left(1+\frac{\left(\frac{1+\sum_{t=1}^T g_t^2}{Z}\right)^{\frac{1}{2}+\frac{Z}{2}}|\w|}{\epsilon}\right)},\right.\right.\nonumber\\
        &\left.\left.\quad\quad\quad\quad\quad\log\left(1+\frac{\left(\frac{1+\sum_{t=1}^T g_t^2}{Z}\right)^{\frac{1}{2}+\frac{Z}{2}}|\w|}{\epsilon}\right)\right)\right]\label{eqn:1dregretbound}
    \end{align}
\end{Theorem}

Note that the original presentation of Theorem \ref{thm:1dolo} in \cite{cutkosky2019matrix} states that $\w$ must satisfy $|\w|\le \frac{1}{2}$ and has no $\w$ dependency inside the logarithm. However, a brief inspection of that result shows that their algorithm was constructed by first obtaining the result of Theorem \ref{thm:1dolo} and then modifying the algorithm to enforce constraints. In order to ease notation in our results, we will just set $Z=1$ and drop the $Z$ dependency in Theorem \ref{thm:1dolo} from all future bounds in the paper. For completeness, we provide a proof of this result in Appendix \ref{sec:1doloproof}.

\section{Adapting to Varying Norms}\label{sec:learning}
In this Section, we show our how to achieve the regret bound (\ref{eqn:varynorm}) in arbitrary convex domains $W$. We decompose the problem into three stages: first, we use FTRL to obtain a bound of the desired form but with suboptimal dependence on $\|\w\|_{T-1}$. Then, we will show how to combine this with a one-dimensional parameter-free algorithm to obtain the desired bound in the case that $W$ is an entire vector space. Finally, we will show how to constrain our algorithm to arbitrary convex $W$. 

Our FTRL algorithm is reminiscent of prior adaptive FTRL methods, but we enforce a special \emph{time varying constraint}. This will make the algorithm much worse on its own, but allow for an overall improvement later. Specifically, suppose we have a sequence of norms $\|\cdot\|_0,\dots\|\cdot\|_{T-1}$ such that $\|x\|_t\ge \|x\|_{t-1}$, and $\frac{1}{2}\|\cdot\|_t^2$ is $\sigma$-strongly-convex with respect to $\|\cdot\|_t$ for all $t$ and $x$. Consider FTRL with regularizers:
\begin{align}
    \psi_t(w)=\left\{ \begin{array}{lr} \frac{1}{\sqrt{\sigma}}\|w\|_t^2\sqrt{1+\sum_{i=1}^t \|g_i\|_{i-1, \star}^2}& \text{if }\|w\|_t\le 1\\
    \infty &\text{if }\|w\|_t> 1
    \end{array}\right.\label{eqn:ftrl}
\end{align}
Then we have the following corollary of Theorem \ref{thm:ftrl}:
\begin{restatable}{Lemma}{thmconstrainedftrl}\label{thm:constrainedftrl}
Let $W$ be a real vector space and $\|\cdot\|_1,\dots,\|\cdot\|_T$ are an increasing sequence of norms on $W$ such that $\frac{1}{2}\|\cdot\|_t$ is $\sigma$-strongly-convex with respect to $\|\cdot\|$. Suppose we run FTRL with regularizers given by (\ref{eqn:ftrl}), and with $g_t$ satisfying $\|g_t\|_{t-1,\star}\le 1$ for all $t$. Then $\|w_t\|_{t-1}\le 1$ for all $t$, and for all $\w$ with $\|\w\|_{T-1}\le 1$, the regret of FTRL is bounded by
\begin{align*}
    R_T(\w)&\le \frac{1}{\sqrt{\sigma}}\left(\|\w\|_{T-1}^2\sqrt{1+\sum_{t=1}^{T-1} \|g_t\|_{t-1,\star}^2} + \sqrt{\sum_{t=1}^T \|g_t\|_{t-1,\star}^2}\right).
\end{align*}
\end{restatable}

\subsection{Unconstrained Domains}\label{sec:unconstrainedvarynorm}
Now, with Lemma \ref{thm:constrainedftrl} in hand, we will proceed to build an algorithm that achieves the bound (\ref{eqn:varynorm}) in the \emph{unconstrained} setting. Our method for the unconstrained setting is very similar to the dimension-free to one-dimensional optimization reduction proposed by \cite{cutkosky2018black}, taking into account the particular dynamics of our FTRL algorithm. Intuitively, we use a one-dimensional parameter-free algorithm to learn a scaling of the FTRL algorithm, which corresponds to a kind of learning rate parameter. The pseudocode for this technique is presented in Algorithm \ref{alg:1dred} below.

\begin{algorithm}
   \caption{Unconstrained Varying Norms Adaptivity}
   \label{alg:1dred}
\begin{algorithmic}
   \STATE{\bfseries Input: } one-dimensional parameter-free online learning algorithm $\ol$, sequence of norms $\|\cdot\|_0,\dots,\|\cdot\|_{T-1}$, real vector space $W$, strong-convexity parameter $\sigma$.
   \STATE Set $\psi_0(x) = \frac{1}{\sqrt{2\sigma}}\|x\|_0^2$.
   \STATE Set $x_1=\argmin_{w\in W} \psi_0(w)$.
   \FOR{$t=1\dots T$}
   \STATE Get $y_t\in \R$ from $\ol$.
   \STATE Output $w_t = y_tx_t$ and get $g_t$.
   \STATE Set $\psi_t(x)=\left\{ \begin{array}{lr} \frac{1}{\sqrt{2\sigma}}\|x\|_t^2\sqrt{1+\sum_{i=1}^t \|g_i\|_{i-1, \star}^2}& \text{if }\|x\|_t\le 1\\
    \infty &\text{if }\|x\|_t> 1\end{array}\right.$
   \STATE Set $x_{t+1} = \argmin_{w\in W} \psi_t(w) +  \sum_{i=1}^{t}\langle g_i, w\rangle$.
   \STATE Send $s_t=\langle g_t, x_t\rangle$ to $\ol$ as the $t$th loss.
   \ENDFOR
\end{algorithmic}
\end{algorithm}
\begin{restatable}{Lemma}{thmonedred}\label{thm:1dred}
Let $R^{1D}_T$ be the regret of the one-dimensional parameter-free algorithm $\ol$. Under the assumptions of Lemma \ref{thm:constrainedftrl}, the regret of Algorithm \ref{alg:1dred} is bounded by:
\begin{align*}
    R_T(\w)&\le R^{1D}_T(\|\w\|_{T-1}) + \frac{2\|\w\|_{T-1}}{\sqrt{\sigma/2}}\sqrt{1+\sum_{t=1}^{T-1} \|g_t\|_{t-1,\star}^2}
\end{align*}
where $R^{1D}_T$ is the regret of $\ol$ on the losses $s_t$. Further, we have $s_t^2\le \|g_t\|_{t-1,\star}^2\le 1$. In particular, if $\ol$ achieves the regret bound (\ref{eqn:1dregretbound}), this yields an overall regret of:
\begin{align*}
    R_T(\w)&\le O\left[\epsilon+\frac{2\|\w\|_{T-1}}{\min(1,\sqrt{\sigma})}\max\left(\sqrt{1+\sum_{t=1}^T \|g_t\|_{t-1,\star}^2\log\left(1+\frac{\sum_{t=1}^T \|g_t\|_{t-1,\star}^2\|\w\|_{T-1}}{\epsilon}\right)},\ \right.\right.\\
    &\quad\quad\quad\quad\quad\quad\quad\left.\left.\log\left(1+\frac{\sum_{t=1}^T \|g_t\|_{t-1,\star}^2\|\w\|_{T-1}}{\epsilon}\right)\right)\right]
\end{align*}
\end{restatable}

\subsection{Adding Constraints}\label{sec:constrainedvarynorm}
Algorithm \ref{alg:1dred} provides a method for obtaining the bound (\ref{eqn:varynorm}) when $W$ is an entire vector space, so in this section we show how to fix the algorithm so that $W$ may be an arbitrary convex domain. We do this by again appealing to a technique from \citep{cutkosky2018black}. This time, we use their Theorem 3, which provides a way to produce constrained algorithms from unconstrained algorithms. The original result considers only the case of a fixed norm and is applied to achieve bounds like (\ref{eqn:constantnorm}). Here we tweak the analysis to consider varying norms as well. The algorithm is presented in Algorithm \ref{alg:varynorm} below, and the analysis achieving (\ref{eqn:varynorm}) is in Theorem \ref{thm:varynorm}.

\begin{algorithm}
   \caption{Varying Norms Adaptivity}
   \label{alg:varynorm}
\begin{algorithmic}
   \STATE{\bfseries Input: } Convex domain $W$ in a real vector space $V$.
   \STATE Define $\Pi_t(v)=\argmin_{w\in W}\|v-w\|_{t-1}$.
   \STATE Define $S_t(v) = \|v-\Pi_t(v)\|_{t-1}$.
   \STATE Initialize Algorithm \ref{alg:1dred} with domain $V$ using the algorithm of Theorem \ref{thm:1dolo} as the base learner.
   \FOR{$t=1\dots T$}
   \STATE Get $t$th output $v_t\in V$ from Algorithm \ref{alg:1dred}.
   \STATE Output $w_t = \Pi_t(v_t)$, and get loss $g_t$.
   \STATE Define $\ell_t(v) = \frac{1}{2}\left(\langle g_t, v\rangle + \|g_t\|_{t-1,\star}S_t(v)\right)$.
   \STATE Let $\hat g_t \in \partial \ell_t(v_t)$, and send $\hat g_t$ to Algorithm \ref{alg:1dred} as the $t$th loss.
   \ENDFOR
\end{algorithmic}
\end{algorithm}
\begin{restatable}{Theorem}{thmvarynorm}\label{thm:varynorm}
Each output $w_t$ of Algorithm \ref{alg:varynorm} lies in $W$, and the regret for any $\w\in W$ is at most:
\begin{align*}
    R_T(\w)&\le  O\left[\epsilon+\frac{\|\w\|_{T-1}}{\min(1,\sqrt{\sigma})}\max\left(\sqrt{1+\sum_{t=1}^T \|g_t\|_{t-1,\star}^2\log\left(1+\frac{\sum_{t=1}^T \|g_t\|_{t-1,\star}^2\|\w\|_{T-1}}{\epsilon}\right)},\ \right.\right.\\
    &\quad\quad\quad\quad\quad\quad\quad\left.\left.\log\left(1+\frac{\sum_{t=1}^T \|g_t\|_{t-1,\star}^2\|\w\|_{T-1}}{\epsilon}\right)\right)\right]
\end{align*}
\end{restatable}

\section{Full-Matrix Bounds}\label{sec:fullmatrix}
The results of the previous section operate with arbitrary norms and in potentially infinite dimensional spaces. In this section and the next, we will specialize to the case $W\subset \R^d$, and show how to obtain so-called ``full-matrix'' or ``preconditioned'' regret bounds. In this section, we will consider the full-matrix regret bound given by (\ref{eqn:fullmatrix}).
% , restated below for convenience:
% \begin{align}
%     R_T(\w)&\le \sqrt{r\sum_{t=1}^T \langle g_t, \w\rangle^2\log(T)}\tag{\ref{eqn:fullmatrix}}
% \end{align}
% Recall here that we defined $G_T=\sum_{t=1}^T g_tg_t^\top$, and $r$ is the dimension of the subspace spanned by the $g_t$, which is also the rank of $G_T$. In general we have $r\le d$.

Up to a factor of $\sqrt{\log(T)}$, this bound is achieved in the case where $W$ is an entire vector space by \cite{cutkosky2018black}, and similar bounds utilizing various extra assumptions are obtained by \cite{kotlowski2019scale, cesa2005second}. When $W$ is not an entire vector space, it seems harder to achieve this bound. However, some progress has been made in certain settings. For example, when $W$ is the probability simplex, \cite{foster2017zigzag} achieves a bound $\sqrt{rT}$, which adapts automatically to $r$. For more general $W$, \cite{koren2017affine} achieves the desired result if their algorithm is tuned with oracle knowledge of $\sum_{t=1}^T \langle g_t, \w\rangle^2$.

Perhaps surprisingly, a relatively straightforward application of Theorem \ref{thm:varynorm} allows us to obtain (\ref{eqn:fullmatrix}), up to a factor of $\log(T)$. Note that this is $\sqrt{\log(T)}$ worse than \cite{cutkosky2018black}, but we are able to handle arbitrary convex domains.

The key idea in our approach is that the norms $\|\cdot\|_t$ used by Algorithm \ref{alg:varynorm} need not be specified ahead of time: so long as $\|\cdot\|_t$ depends only on $g_1,\dots,g_t$, it is still possible to run the algorithm. Next, observe that $\sum_{t=1}^T \langle g_t, \w\rangle^2$ can be viewed as $\|\w\|_{G_T}^2$, where we recall that $\|\cdot\|_{G_T}$ is the norm induced by $G_T$: $\|x\|_{G_T}^2=x^\top G_Tx$. Inspired by these observations, our approach is to run Algorithm \ref{alg:varynorm} using norms $\|\cdot\|_t=\|\cdot\|_{G_t}$. The algorithm is analyzed in Theorem \ref{thm:fullmatrix} below.
% \begin{algorithm}
%   \caption{Full-Matrix Adaptivity}
%   \label{alg:fullmatrix}
% \begin{algorithmic}
%   \STATE{\bfseries Input: } Convex domain $W$.
%   \STATE Initialize Algorithm \ref{alg:varynorm}.
%   \STATE Set $\|x\|_0^2= 2\|x\|_{\mathbf{2}}^2$ where $\|\cdot\|_{mathbf{2}}$ is the $L_2$ norm.
%   \FOR{$t=1\dots T$}
%   \STATE Get $t$th output $w_t\in W$ from Algorithm \ref{alg:varynorm}.
%   \STATE Output $w_t$.
%   \STATE Get loss $g_t$.
%   \STATE Set $G_t = \sum_{i=1}^t g_t g_t^\top$.
%   \STATE Define $\|x\|_t^2 = \|x\|_{2I+G_t}^2=x^\top(2I+G_t)x$, where $I$ is the $d\times d$ identity matrix.
%   \STATE Send $g_t$ and $\|\cdot\|_t$ to Algorithm \ref{alg:varynorm} as $t$th loss and norm respectively,
%   \ENDFOR
% \end{algorithmic}
% \end{algorithm}

\begin{restatable}{Theorem}{thmfullmatrix}\label{thm:fullmatrix}
Suppose $g_t$ satisfies $\|g_t\|\le 1$ for all $t$ where $\|\cdot\|$ is any norm such that $\frac{1}{2}\|\cdot\|^2$ is $\sigma$-strongly convex with respect to $\|\cdot\|$. Let $G_t = \sum_{i=1}^t g_ig_i^\top$ and let $r$ be the rank of $G_T$. Suppose we run Algorithm \ref{alg:varynorm} with $\|x\|_t^2=\|x\|^2+x^\top(I + G_t)x$, where $I$ is the identity matrix. Then we obtain regret $R_T(\w)$ bounded by:
\begin{align*}
     &O\left[\sqrt{\frac{\|\w\|^2+\|\w\|_{\mathbf{2}}^2+\sum_{t=1}^T\langle g_t, \w\rangle^2}{\min(\sigma, 1)}}\max\left(\log\left(1+\frac{r\log(T)\sqrt{\|\w\|^2+\|\w\|_{\mathbf{2}}^2+\sum_{t=1}^T\langle g_t, \w\rangle^2}}{\epsilon}\right),\right.\right.\\
    &\quad\quad\quad\quad\quad\left.\left.\sqrt{r\log(T)\log\left(1+\frac{r\log(T)\sqrt{\|\w\|^2+\|\w\|_{\mathbf{2}}^2+\sum_{t=1}^T\langle g_t, \w\rangle^2}}{\epsilon}\right)}\right)+\epsilon\right]
\end{align*}
\end{restatable}
\begin{proof}
We have $\|x\|_t^2=\|\x\|^2x^\top(I+G_t)x=\|x\|+\|x\|_{\mathbf{2}}^2+\sum_{i=1}^t\langle g_t, x\rangle^2$,
so that $\|\cdot\|_t$ is increasing in $t$. Further, since $\|x\|_{t-1}\ge \|x\|$, we must have $\|g_t\|_{t-1,\star}\le 1$ for all $t$.  Next, observe that since $\|g_t\|_\star\le 1$, we have
\begin{align*}
\|x\|^2+\|x\|_{\mathbf{2}}^2+\sum_{i=1}^{t-1}\langle g_i, x\rangle^2 \ge\|x\|_{\mathbf{2}}^2+\sum_{i=1}^{t}\langle g_i, x\rangle^2= x^\top(I+G_t)x
\end{align*}
Therefore, we have $\|g_t\|_{t-1,\star} \le g_t^\top (I+ G_t)^{-1}g_t$. Now recall that for any PSD matrix $M$, $\frac{1}{2}x^\top Mx$ is 1-strongly convex with respect to the norm $\sqrt{x^\top Mx}$. Therefore, by Lemma \ref{thm:addnorm}, we have that $\frac{1}{2}\|x\|_t^2$ is $\min(\sigma, 1)$-strongly convex with respect to $\|\cdot\|_t$ so that we have satisfied all the hypotheses of Theorem \ref{thm:varynorm}. Finally, before we apply Theorem \ref{thm:varynorm}, we need to analyze
\begin{align*}
    \sum_{t=1}^T \|g_t\|_{t-1,\star}^2 \le \sum_{t=1}^T g_t^\top(I+G_t)^{-1} g_t
    &\le \log\left(\frac{\det(I+\sum_{i=1}^t g_tg_t^\top)}{\det(I)}\right)\\
    &\le \text{rank}(G_T)\log(T+1)
\end{align*}
where we have applied Lemma 11 of \cite{hazan2007logarithmic}. The result now follows from Theorem \ref{thm:varynorm}.
\end{proof}

Note that for concreteness, if we set $\|\cdot\|=\|\cdot\|_{\mathbf{2}}$ in the above bound, then the norms $\|\cdot\|_t$ become the familiar matrix-based norm $\|x\|_t= \sqrt{x^\top(2I + G_t)x}$. We have opted to leave the more general formulation in place to allow for $g_t$ that are not bounded in the $L_2$ norm.

% As a specific example of the advantage of our new bound, consider the \emph{low-rank experts problem} \citep{hazan2016online}. This is an online linear optimization problem in which $W$ is the probability simplex in $\R^d$, the losses $g_t$ satisfy $\|g_t\|_\infty\le 1$ and $g_t$ span a low-dimensional space. In this scenario, the worst-case $\w$ must necessarily be one of the $d$ vertices of the simplex. If we use $\|\cdot\|=\|\cdot\|_{\log(d)}$ in the above, then up to a logarithmic factor, our regret with respect to the $i$th vertex is at most $\sqrt{r \sum_{t=1}^T g_{t,i}^2}+\sqrt{\log(d)}$. This adds adaptivity to prior bounds, which guarantee only $\sqrt{rT}$ \citep{foster2017zigzag, koren2017affine}.

\section{Full-Matrix Adagrad with Oracle Tuning}\label{sec:adagrad}
In this section we consider a different kind of full-matrix bound inspired by the full-matrix AdaGrad algorithm \citep{duchi10adagrad}. Full-matrix AdaGrad can be described as FTRL using regularizers:\footnote{In \cite{duchi10adagrad}, this version of AdaGrad is called the Primal-Dual update version.}
\begin{align*}
    \psi_t(x) = \frac{1}{\eta}\langle x, (I+G_t)^{1/2}, x\rangle
\end{align*}
where $\eta$ is a scalar learning rate parameter that must be set by the user. $(I+G_t)^{1/2}$ indicates the symmetric positive-definite matrix square-root of $I+G_t$, which exists since $I+G_t$ is a symmetric positive-definite matrix. This algorithm is empirically very successful, in spite of the significant computational overhead coming from manipulating the $d\times d$ matrix $G_t$. Indeed, much work has gone into providing approximate versions of this algorithm that reduce the computation load while still retaining some of the performance benefits \citep{gupta2018shampoo,agarwal2019efficient,chen2019extreme}. Prior analyses of full-matrix AdaGrad considers domains $W$ with finite diameter $D=\sup_{x,y\in W}\|x-y\|_{\mathbf{2}}$, and suggests setting $\eta=O(D)$ to obtain a regret bound of:
\begin{align*}
    R_T(\w)&\le O(D\trace(G_T^{1/2}))
\end{align*}
However, by linearity of trace and concavity of square root, we have:
\begin{align*}
    D\trace(G_T^{1/2})\ge D\sqrt{\trace G_T}=D\sqrt{\sum_{t=1}^T \|g_t\|_{\mathbf{2}}^2}
\end{align*}
The bound $R_T(\w)\le D\sqrt{\sum_{t=1}^T \|g_t\|_{\mathbf{2}}^2}$ can be achieved by simple (and fast) online gradient descent with a scalar learning rate, $w_{t+1}=w_t - \frac{Dg_t}{\sqrt{\sum_{i=1}^t \|g_t\|_{\mathbf{2}}^2}}$, so the prior regret bound of full-matrix AdaGrad does not appear to show any benefit gained by the extra matrix computations. This poses a mystery: since the actual algorithm is so effective, it seems we are missing something in the analysis. We propose a possible explanation for this quandary. The main idea is that, in practice, the theoretical guidance to set $\eta=O(D)$ is rarely used. Instead, $\eta$ is tuned via manually checking different values to find which is empirically best. Thus, if we could show that full-matrix AdaGrad achieves gains with an oracle-tuning for $\eta$, this might explain the improved performance in practice.

To this end, recall that from Theorem \ref{thm:ftrl} we can write the regret of full-matrix AdaGrad as:
\begin{align*}
    R_T(\w)&\le O\left(\frac{ \w^\top (I+G_{T})^{1/2}\w}{\eta} +\eta \sum_{t=1}^T g_t^\top (I+G_{t-1})^{-1/2}g_t\right)\le O\left(\frac{ \w^\top G_T^{1/2}\w}{\eta} +\eta \trace(G_T^{1/2})\right)
\end{align*}
where the second inequality is due to Lemma 10 of \cite{duchi10adagrad}, and we have ignored the dependence $I$ for simpler exposition. Then it is clear that with the optimal tuning of $\eta = O\left(\sqrt{\frac{\langle \w, G_T^{1/2}\w\rangle}{ \trace(G_T^{1/2})}}\right)$, we obtain regret bound of (\ref{eqn:adagradbound}).
% , restated below:
% \begin{align}
%     R_T(\w)&\le O\left(\sqrt{\langle \w, G_T^{1/2}\w\rangle\trace (G_T^{1/2})}\right) \tag{\ref{eqn:adagradbound}}
% \end{align}
In order to appreciate the potential of this bound, let us construct a particular sequence of $g_t$s and evaluate the bound. We will compare the bound (\ref{eqn:adagradbound}) to (\ref{eqn:fullmatrix}) as well as to (\ref{eqn:constantnorm}) with the $L_2$ norm. Our example will illustrate that (\ref{eqn:adagradbound}) can in some sense adapt to the case that $G_T$ is full-rank but ``approximately low rank'', while the analysis of the full-matrix algorithm in Section \ref{sec:fullmatrix} does not obviously allow for such behavior.

Let $v_1,\dots,v_d$ be an orthonormal basis for the $d$-dimensional vector space containing $W$. Assume $d$ is a perfect square and $T=2d+2k\sqrt{d}$ for some integer $k$. For the first $d$ rounds, $g_t=v_t$ and for the second $d$ rounds $g_{d+t}=-v_{t}$. For the remaining rounds, we write $t=i+j\sqrt{d}+2d$ for $j\in \Z$ and $1\le i\le \sqrt{d}$, and set $g_{t}=\frac{1}{\sqrt{d}}v_d + \left((-1)^{j}\sqrt{1-\frac{1}{d}}\right)v_{i}$. Intuitively, the losses are cycling with alternating signs through the first $\sqrt{d}$ basis vectors, but always maintain a small positive component in the direction of $v_d$. Notice that since $T-2d$ is a multiple of $2\sqrt{d}$, the alternating signs imply that $\sum_{t=1}^T g_t$ is a positive scalar multiple of $v_d$. Consider $\w=-v_d$. Then, we have:
% we have:
% \begin{align*}
%     \sum_{t=1}^T g_t&= \frac{T-2d}{\sqrt{d}}v_d\\
%     \left\langle \w, \left(\sum_{t=1}^T g_tg_t^\top\right)^{1/2} \w\right\rangle&=O(\sqrt{T/d})\\
%     \sum_{t=1}^T \langle g_t, \w\rangle^2&=O(T/d)\\
%     \trace\sqrt{\sum_{t=1}^T g_tg_t^\top}&=O(d + d^{1/4}\sqrt{T})\\
%     \text{rank}\left(\sum_{t=1}^T g_tg_t^\top\right)&=d
% \end{align*}
% So, discarding log factors for simplicity, let us compare the bounds (\ref{eqn:constantnorm}), (\ref{eqn:fullmatrix}), and (\ref{eqn:adagradbound}). 
\begin{align*}
    \|\w\|_2\sqrt{\sum_{t=1}^T\|g_t\|_2^2}&=O(\sqrt{T})\\
    \sqrt{\text{rank}(G_T)\sum_{t=1}^T \langle g_t, w_t\rangle^2}&=O(\sqrt{T})\\
    \sqrt{\langle \w, G_T^{1/2}\w\rangle\trace (G_T^{1/2})}&=O\left(\sqrt{T/d^{1/4}+\sqrt{dT}}\right)
\end{align*}
In this case, the trace of $\sqrt{\sum_{t=1}^T g_tg_t^\top}$ captures the fact that even though the $g_t$ span $d$ dimensions, they are approximately contained in $\sqrt{d}$ dimensions. This allows bound (\ref{eqn:adagradbound}) to perform much better than either of the other bounds. In contrast, if the example is modified so that the first $2d$ rounds only cycle between the first $\sqrt{d}$ basis vectors, we would have $\text{rank}(G_T)=\sqrt{d}$ and so the full-matrix bound (\ref{eqn:fullmatrix}) is the best. Finally, if we increase the component on $v_d$ in each round to, for example, $\frac{1}{\sqrt{2}}$, then the bound (\ref{eqn:constantnorm}) is the smallest. Therefore none of the bounds uniformly dominates the others.

To gain a little more intuition for what the bound \ref{eqn:adagradbound} means, let us investigate the worst-case performance of the bounds (\ref{eqn:constantnorm}), (\ref{eqn:fullmatrix}) and (\ref{eqn:adagradbound}) over all $\w$ with $\|\w\|_{\mathbf{2}}\le 1$. To this end, write $T_{\eff}=\sum_{t=1}^T \|g_t\|_{\mathbf{2}}^2$ and let $\lambda_{\max} = \sup_{\|\w\|\le 1} \sum_{t=1}^T \langle g_t, \w\rangle^2$. Then we clearly have (\ref{eqn:constantnorm}) is $O(\sqrt{T_{\eff}})$ while the bound (\ref{eqn:fullmatrix}) is at most $O(\sqrt{r \lambda_{\max}})$. On the other hand, by Cauchy-Schwarz inequality we have $\trace(G_T^{1/2})\le \sqrt{r_{\eff} T_{\eff}}$ where $r_{\eff}\le r$ is some ``effective rank'' that might be much lower than the true rank $r$. With this notation, we have that the bound (\ref{eqn:adagradbound}) is at most $(\lambda_{\max} r_{\eff} T_{\eff})^{1/4}$. Thus, we see that the new bound is at most the geometric mean of the bounds (\ref{eqn:constantnorm}) and (\ref{eqn:fullmatrix}), but could potentially be much lower if the effective rank $r_{\eff}$ is smaller than $r$.

\subsection{Achieving the Optimal Full-Matrix AdaGrad Bound}

Now that we see there is some potential advantage to a bound like (\ref{eqn:adagradbound}), we will show how to obtain the bound without manually tuning $\eta$ using our framework. The approach is very similar to how we obtained the bound (\ref{eqn:fullmatrix}): we run Algorithm \ref{alg:varynorm} and in round $t$ we set $\|\cdot\|_t = \|\cdot\|_{G_t^{1/2}}$. With this setting, the desired bound is an almost immediate consequence of Theorem \ref{thm:varynorm}:

% \begin{algorithm}
%   \caption{AdaGrad-style Full-Matrix Adaptivity}
%   \label{alg:adagrad}
% \begin{algorithmic}
%   \STATE{\bfseries Input: } Convex domain $W$.
%   \STATE Initialize Algorithm \ref{alg:varynorm}.
%   \STATE Set $\|\cdot\|_0= \|\cdot\|_{\mathbf{2}}$.
%   \FOR{$t=1\dots T$}
%   \STATE Get $t$th output $w_t\in W$ from Algorithm \ref{alg:varynorm}.
%   \STATE Output $w_t$.
%   \STATE Get loss $g_t$.
%   \STATE Set $G_t = \sum_{i=1}^t g_t g_t^\top$.
%   \STATE Define $\|\cdot\|_t = \|\cdot\|_{(I+G_t)^{1/2}}$, where $I$ is the $d\times d$ identity matrix.
%   \STATE Send $g_t$ and $\|\cdot\|_t$ to Algorithm \ref{alg:varynorm} as $t$th loss and norm respectively,
%   \ENDFOR
% \end{algorithmic}
% \end{algorithm}
\begin{restatable}{Theorem}{thmadagrad}\label{thm:adagrad}
Suppose $W\subset \R^d$ and $g_t$ satisfies $\|g_t\|_{\mathbf{2}}\le 1$ for all $t$. Let $G_t= \sum_{i=1}^t g_ig_i^\top$. Define $\|\cdot\|_t$ be $\|x\|_t^2=x^\top (I+G_t)^{1/2} x$. Then the regret of Algorithm \ref{alg:varynorm} using these norms is bounded by:
\begin{align*}
    R_T(\w)&\le \tilde O\left(\sqrt{(\|\w\|_{\textbf{2}}^2 + \w^{\top}G_T^{1/2}\w)\trace\left(G_T^{1/2}\right)}\right)
\end{align*}
where the $\tilde O$ notation hides a logarithmic dependency on $\trace\left(G_T^{1/2}\right)\sqrt{\|\w\|_{\textbf{2}}^2 + \w^{\top}G_T^{1/2}\w}$.
\end{restatable}
This Theorem recovers the desired bound (\ref{eqn:adagradbound}) up to log factors. Moreover, it is possible to interpret the operation of the algorithm as in some rough sense ``learning the optimal learning rate'' required for the original AdaGrad algorithm to achieve this bound.
\begin{proof}
Observe that since $\|g_t\|_{\mathbf{2}}\le 1$, we have $\|g_t\|_{t-1,\star}=\|g_t\|_{(I+G_{t-1})^{-1/2}}\le \|g_t\|_{\mathbf{2}}\le 1$ so that the hypotheses of Theorem \ref{thm:varynorm} are satisfied. In order to complete the analysis we need only calculate:
\begin{align*}
    \sum_{t=1}^T \|g_t\|_{t-1,\star}^2&=\sum_{t=1}^T g_t^\top (I+G_{t-1})^{-1/2}g_t\le \sum_{t=1}^T g_t^\top G_t^{-1/2}g_t\le 2\trace(G_T^{1/2})
\end{align*}
Here, in the first inequality, we mildly abuse of notation to indicate the pseudo-inverse of $G_t^{1/2}$ as $G_{t-1}^{-1/2}$. The inequalities then follow from \cite{duchi10adagrad} Lemmas 9 and 10.

Finally, observe that $(I+G_T)^{1/2}\preceq I + G_T^{1/2}$, and apply Theorem \ref{thm:varynorm} to obtain the result.
\end{proof}

\section{Scale-Invariant Algorithms}\label{sec:scale}
In this section we consider the online linear supervised learning problem in the unconstrained setting, a slight modification of the general online convex optimization paradigm. Now, the losses $\ell_t(w)$ take the form $\ell_t(w)=c_t(\langle f_t, w\rangle)$ where $c_t:\R\to\R$ is a 1-Lipschitz convex function, $f_t\in \R^d$ is called a ``feature vector'', and $f_t$ is revealed to the learner \emph{before} the learner commits to the choice of $w_t$. 
%The values $\langle f_t, w\rangle$ are called the ``predictions'' of the learner.
%This setting models many practical problems of interest. For example, in logistic regression, $x_t$ represents the features of the $t$th example, and $c_t(y)=\log(1+\exp(y))-(1-y_t)y$ where $y_t$ is true label of the example.
A desirable property for an algorithm in this setting is to be \emph{scale-invariant}, which means the values $\langle f_t, w_t\rangle$ should be unchanged if each component $f_{t,i}$ of the features is rescaled by some unknown value $m_i$ (the $c_t$ functions remain the same). This corresponds to robustness to some kind of ``unit-mismatch'' in the features.
Further, scale-invariance can also be employed in the framework of \cite{cutkosky2019artificial} to produce an algorithm that adapts to an unknown bound on $\|g_t\|_\star$ as well as the unknown value of $\|\w\|$.\footnote{Recall that we have relied on $\|g_t\|_\star\le 1$ in our present analysis.} In this case, the scale-invariant property eliminates a logarithmic dependence on the first loss norm $\|g_1\|_\star$ that is incurred by the original analysis.
%For example, if $f_t\in \R^d$ is a vector whose coordinates represent $d$ distance measurements of various kinds, the learner should not care if the first coordinate is measured in centimeters and the last coordinate in light-years. Note that if $f_t$ is transformed to $Mf_t$, we should replace the comparator $\w$ by $(M^{-1})^\top$. Our goal will be to obtain the full-matrix bound (\ref{eqn:fullmatrix}) with a scale-invariant algorithm.

Several prior works deal with this problem. The first we are aware of is \cite{ross2013normalized}, who considered a bounded diameter setting. Later, \cite{kempka2019adaptive} improved upon these results to allow for unbounded domains. The more general case of invariance to arbitrary linear transformations was studied by \cite{luo2016efficient} and \cite{kotlowski2019scale} - we provide some results in this setting using our framework in Appendix \ref{sec:scalefull}.

% studied the case of arbitrary $M$ and unbounded losses, achieving a bound similar to (\ref{eqn:fullmatrix}), but replacing the term $\sum_{t=1}^T \langle g_t, \w\rangle^2$ with $\sum_{t=1}^T \langle f_t, \w\rangle^2$. 
% Since $g_t = \nabla_t f_t$ and $c_t$ is 1-Lipschitz, this latter bound is always worse.

Our approach is again a relatively straightforward application of Algorithm \ref{alg:varynorm}. The key idea is that it is easy to make the FTRL algorithm used in Algorithm \ref{alg:varynorm} scale-invariant. Then, the losses sent to the one-dimensional algorithm will be unchanged by scaling, so that the entire algorithm is scale-invariant. Our algorithm and analysis are presented in Algorithm \ref{alg:diagscale} and Theorem \ref{thm:diagscale}.

% \subsection{Diagonal Scale-Invariance}\label{sec:scalediag}
% In order to achieve diagonal scale invariance, we first observe the commonly used trick that a $d$-dimensional problem can be solved by running $d$ separate one-dimensional optimizers in parallel:
% \begin{align*}
%     \sum_{t=1}^T \langle g_t, w_t-\w\rangle =\sum_{i=1}^d\sum_{t=1}^T g_{t,i}(w_{t,i}-\w_i)
% \end{align*}
% where the subscript $i$ indicates the $i$th component of a vector. Given this, if we can achieve a scale-invariant algorithm for $f_t\in \R$, we run an independent copy of this algorithm on each dimension of the original problem to obtain invariance to any diagonal transformation of the features. Our method for achieving this one-dimensional scale invariance is quite simple: let $m_t=\max_{t'\le t} |f_{t'}|$ and set $\|\cdot\|_t = m_{t+1}|\cdot|$. Notice that since $f_{t+1}$ is available before we commit to $w_{t+1}$, it is permissible to use this norm. It is then relatively straightforward to check that Algorithm \ref{alg:varynorm} is scale-free and achieves a regret bound like (\ref{eqn:constantnorm}).

\begin{algorithm}
   \caption{Diagonal Scale-Invariance}
   \label{alg:diagscale}
\begin{algorithmic}
   \STATE Initialize $d$ one-dimensional copies of Algorithm \ref{alg:varynorm}.
   \FOR{$t=1\dots T$}
   \FOR{$i=1\dots d$}
   \STATE $m_{t,i}=\sup_{t'\le t} |f_{t',i}|$.
   \STATE Set $\|x\|_{t-1}=m_{t, i}|x|$ and send $\|x\|_{t-1}$ to the $i$th copy of Algorithm \ref{alg:varynorm} as the $t-1$th norm.
   \STATE Get $t$th output $w_{t,i}$ from $i$th copy of Algorithm \ref{alg:varynorm}.
   \ENDFOR
   \STATE Output $w_t=(w_{t,1},\dots,w_{t,d})$ and get loss $\ell_t(\cdot)=c_t(\langle f_t, \cdot\rangle)$.
   \STATE Set $\nabla_t \in \partial c_t(\langle f_t, w_t\rangle)$ and $g_t=\nabla _t f_t\in\partial \ell_t(w_t)$.
   \STATE For each $i$, send $g_{t,i}$ to $i$th copy of Algorithm \ref{alg:varynorm} as $t$th loss.
   \ENDFOR
\end{algorithmic}
\end{algorithm}

\begin{restatable}{Theorem}{thmdiagscale}\label{thm:diagscale}
Suppose $|\nabla_t|\le 1$ for all $t$. Then
Algorithm \ref{alg:diagscale} is scale-invariant with respect to any invertible diagonal linear transformation and achieves regret:
\begin{align*}
    R_T(\w)&\le O\left[d\epsilon+ \sum_{i=1}^d |\w_i|\sqrt{M_{T,i}^2\sum_{t=1}^T \nabla_t^2 \frac{f_{t,i}^2}{m_{t,i}^2}\log\left(1+\frac{\sum_{t=1}^T \nabla_t^2\frac{f_{t,i}^2}{m_{t,i}^2}|\w_i|}{\epsilon}\right)}\right]
\end{align*}
where we define $\frac{0}{0}=0$ and we assume Algorithm \ref{alg:varynorm} will output $0$ for rounds in which $\|\cdot\|_{t-1}$ is 0.
\end{restatable}
Let us contrast our result in Theorem \ref{thm:diagscale} with the regret bounds for the same setting in \cite{kempka2019adaptive}. This prior work achieves a similar result, but instead of $M_T^2\sum_{t=1}^T \delta_t^2 \frac{f_t^2}{m_t^2}$, the bound depends only on $M_T^2+\sum_{t=1}^T \delta_t^2 f_t^2$, which may be better if the $f_t$ are arranged in an adversarially increasing manner. However, our bound improves the logarithm term, moving from $O(\log(T))$ to $O(\sqrt{\log(T)})$. We leave open whether it is possible to obtain the best of both worlds in this setting.

%Note that both algorithms achieve regret $\tilde O\left(\|\w\|_{\mathbf{2}}\sqrt{\sum_{t=1}^T \nabla_t^2}\right)$, which can imply small regret when the losses are smooth and $\sum_{t=1}^T \ell_t(\w)$ is small or has low variance, as outlined in \cite{srebro2010smoothness, levy2017online, cutkosky2019anytime}.

\section{Conclusion}
We have introduced an online linear optimization algorithm that achieves the regret bound
\begin{align*}
    R_T(\w)&\le \tilde O\left(\|\w\|_{T-1}\sqrt{\sum_{t=1}^T \|g_t\|_{t-1,\star}^2}\right)
\end{align*}
for any increasing sequence of norms $\|\cdot\|_0,\dots,\|\cdot\|_{T-1}$, so long as $\|\cdot\|_t$ depends only on $g_1,\dots,g_t$. Our approach uses a particular FTRL analysis combined with a one-dimensional parameter-free algorithm to learn the optimal learning rate for the FTRL algorithm. This general result can be used to obtain improved full-matrix algorithms. In particular, we provided an alternative regret analysis of the full-matrix AdaGrad algorithm, which takes into account the reality that in practice the learning rate is tuned manually. This yields a bound that for the first time shows a strong theoretical advantage to full-matrix AdaGrad, helping to explain its empirical success. Our new framework allows us to achieve this regret bound automatically, \emph{without} requiring manual tuning. Finally, we presented an application of our techniques to scale-invariant supervised learning.

Our results raise several interesting open questions. Firstly, our full-matrix regret bound seems to be a factor of $\sqrt{\log(T)}$ worse than the best rate in the unconstrained case, suggesting that there is some room to improve our algorithm or analysis. Second, one might interpret our overall technique as a way to ``learn the learning rate'' in FTRL algorithms for which the regularizers are minimized at 0. This intuition is reminiscent of the MetaGrad algorithm \citep{vanervan2016metagrad}, which intuitively tunes the learning rate of a mirror-descent-like algorithm to obtain regret $\sqrt{d\sum_{t=1}^T \langle g_t, w_t -\w\rangle^2}$, at the cost of an $O(\log(T))$ slowdown in runtime. This suggests the question: can we generalize our techniques to efficiently learn the learning rate for other methods such as Mirror Descent, or FTRL with non-centered regularizers?

% Acknowledgments---Will not appear in anonymized version
% \acks{We thank a bunch of people.}
\bibliographystyle{unsrtnat}
\bibliography{all}

\appendix
\section{Appendix Organization}

This appendix is organized as follows: in Section \ref{sec:constrainedftrlproof}, \ref{sec:1dredproof} and \ref{sec:varynormproof} we provide the missing proofs of Theorems \ref{thm:constrainedftrl}, \ref{thm:1dred} and \ref{thm:varynorm}. In Section \ref{sec:detailedfullmatrix} we provide detailed version of Theorems \ref{thm:fullmatrix} and \ref{thm:adagrad} containing all constants. In Section \ref{sec:1doloproof} we provide a version of Theorem \ref{thm:1dolo} with all constants for completeness. Finally, in Section \ref{sec:scaleproofs} we provide proofs for our scale-invariant algorithms.

\section{Proof of Theorem \ref{thm:constrainedftrl}}\label{sec:constrainedftrlproof}
In this section we provide the missing proof of Theorem \ref{thm:constrainedftrl}, restated below:
\thmconstrainedftrl*
\begin{proof}
To begin, observe that since $\psi_t(w)=\infty$ for $\|w\|_t>1$, the definition of the FTRL update implies $\|w_{t+1}\|_t\le 1$. So now it remains only to show the regret bound.

By the $\sigma$-strong-convexity of $\frac{1}{2}\|\cdot\|_t^2$, we have that $\psi_t$ is $\sqrt{2\sigma +2\sigma\sum_{i=1}^t \|g_i\|_{i-1, \star}^2}$-strongly convex with respect to $\|\cdot\|_t$. Further, since $\|\cdot\|_t$ is increasing with $t$, $\psi_t$ is increasing as well. Therefore direct application of Theorem \ref{thm:ftrl} yields:
\begin{align*}
    R_T(\w)&\le \psi_{T-1}(\w) + \sum_{t=1}^T \frac{\|g_t\|_{t-1,\star}^2}{2\sqrt{\sigma+\sigma\sum_{i=1}^{t-1} \|g_i\|_{i-1, \star}^2}}
\end{align*}
Now we recall the following consequence of concavity of the square root function (see \cite{auer2002adaptive, duchi10adagrad} for proofs): for any sequence non-negative numbers $x_1,\dots,x_T$ we have
\begin{align*}
    \sum_{t=1}^T \frac{x_t}{\sqrt{\sum_{i=1}^t x_t}}\le 2\sqrt{\sum_{t=1}^T x_t}
\end{align*}
Using this observation, and the fact that $\|g_t\|_{t-1,\star}\le 1$, we have
\begin{align*}
    \sum_{t=1}^T \frac{\|g_t\|_{t-1,\star}^2}{2\sqrt{\sigma+\sigma\sum_{i=1}^{t-1} \|g_i\|_{i-1, \star}^2}}&\le \sum_{t=1}^T \frac{\|g_t\|_{t-1,\star}^2}{2\sqrt{2\sigma}\sqrt{\sum_{i=1}^t \|g_i\|_{i-1, \star}^2}}\\
    &\le \sqrt{\frac{1}{\sigma} \sum_{t=1}^T \|g_t\|_{t-1,\star}^2}
\end{align*}
And now the final bound follows by inserting the definition of $\psi_{T-1}$.
\end{proof}

\section{Proof of Theorem \ref{thm:1dred}}\label{sec:1dredproof}
In this section we provide the missing proof of Theorem \ref{thm:1dred}, restated below:
\thmonedred*
\begin{proof}
First, by Lemma \ref{thm:constrainedftrl}, we have $\|x_t\|_{t-1}\le 1$, so that $\langle g_t, x_t\rangle\le \|g_t\|_{t-1,\star}\|x_t\|_{t-1}\le \|g_t\|_{t-1,\star}\le 1$. Next, we use an argument from \cite{cutkosky2018black}:
\begin{align*}
    \sum_{t=1}^T\langle g_t, w_t - \w\rangle&=\sum_{t=1}^T \langle g_t, y_t x_t - \w\rangle\\
    &=\sum_{t=1}^T \langle g_t, x_t\rangle(y_t -\|\w\|_{T-1}) + \|\w\|_{T-1}\sum_{t=1}^T \langle g_t, x_t - \w/\|\w\|_{T-1}\rangle\\
    &= R^{1D}_T(\|\w\|_{T-1}) + R^{FTRL}_T(\w/\|\w\|_{T-1})
\end{align*}
where $R^{FTRL}_T$ is the regret of FTRL. Since $\left\|\frac{\w}{\|\w\|_{T-1}}\right\|_{T-1}=1$, Lemma \ref{thm:constrainedftrl} tells us:
\begin{align*}
    R^{FTRL}_T(\w/\|\w\|_{T-1})&\le \frac{2}{\sqrt{\sigma}}\sqrt{1+\sum_{t=1}^{T-1} \|g_t\|_{t-1,\star}^2}
\end{align*}
and so we have shown the first regret bound. For the second, observe that $|s_t|\le \|g_t\|_{t-1,\star}\le 1$, so we can apply the regret bound of Theorem \ref{thm:1dolo}. Specifically, if we pull the constants from Theorem \ref{thm:fullonedolo}, we obtain:
\begin{align*}
    R_T(\w)&\le \epsilon+ 2\|\w\|_{T-1}\max\left[\sqrt{\left(3+3\sum_{t=1}^T \|g_t\|_{t-1,\star}^2\right)\log\left(e+\frac{\|\w\|_{T-1}(6 + 11\sum_{t=1}^T \|g_t\|_{t-1,\star}^2)}{\epsilon}\right)}, \right.\\
    &\quad\quad\quad\quad\quad\left. 2 \log\left(e+\frac{\|\w\|_{T-1}(6 + 11\sum_{t=1}^T \|g_t\|_{t,-1\star}^2)}{\epsilon}\right)\right] \\
    &\quad\quad+\frac{2\|\w\|_{T-1}}{\sqrt{\sigma}}\sqrt{1+\sum_{t=1}^{T-1} \|g_t\|_{t-1,\star}^2}
\end{align*}
\end{proof}

\section{Proof of Theorem \ref{thm:varynorm}}\label{sec:varynormproof}
In this section, we provide the missing proof of Theorem \ref{thm:varynorm}, restated below:
\thmvarynorm*
\begin{proof}
The proof is nearly identical to that \cite{cutkosky2018black} Theorem 3 - we simply observe that none of the steps in their proof required a fixed norm, and reproduce the argument for completeness. From \cite{cutkosky2018black} Proposition 1, we have that $S_t$ is convex and Lipschitz with respect to $\|\cdot\|_{t-1}$ for all $t$. Therefore we have $\ell_t$ is also convex and $\|g_t\|_{t-1,\star}$-Lipschitz with respect to $\|\cdot\|_{t-1}$. Therefore we have $\|\hat g_t\|_{t-1,\star}\le \|g_t\|_{t-1,\star}$.
\begin{align*}
    \sum_{t=1}^T \langle g_t, w_t - \w\rangle&=\sum_{t=1}^T \langle g_t, v_t \rangle +\langle g_t, w_t -v_t\rangle - \langle g_t, \w\rangle\\
    &\le \sum_{t=1}^T \langle g_t,v_t\rangle + \|g_t\|_{t-1,\star}\|w_t - v_t\|_{t-1} - \langle g_t, \w\rangle\\
    &=2\sum_{t=1}^T \ell_t(v_t) - \ell_t(\w)\\
    &\le 2\sum_{t=1}^T \langle \hat g_t, v_t - \w\rangle
\end{align*}
Now since $\|\hat g_t\|_{t-1,\star}\le \|g_t\|_{t-1,\star}\le 1$, we have that $\sum_{t=1}^T \langle \hat g_t, v_t - \w\rangle$ is simply the regret of the unconstrained Algorithm \ref{alg:1dred} and so the Theorem follows. Specifically, if we again substitute in the result of Theorem \ref{thm:fullonedolo} to get all constants, we obtain:
\begin{align*}
    R_T(\w)&\le \epsilon+ 2\|\w\|_{T-1}\max\left[\sqrt{\left(3+3\sum_{t=1}^T \|g_t\|_{t-1,\star}^2\right)\log\left(e+\frac{\|\w\|_{T-1}(6 + 11\sum_{t=1}^T \|g_t\|_{t-1,\star}^2)}{\epsilon}\right)}, \right.\\
    &\quad\quad\quad\quad\quad\left. 2 \log\left(e+\frac{\|\w\|_{T-1}(6 + 11\sum_{t=1}^T \|g_t\|_{t,-1\star}^2)}{\epsilon}\right)\right] \\
    &\quad\quad+\frac{2\|\w\|_{T-1}}{\sqrt{\sigma}}\sqrt{1+\sum_{t=1}^{T-1} \|g_t\|_{t-1,\star}^2}
\end{align*}
\end{proof}

\section{Detailed Full-Matrix Bounds with Constants}\label{sec:detailedfullmatrix}
In this section, we show a more detailed proof of Theorems \ref{thm:fullmatrix} and \ref{thm:adagrad} that includes all constant factors and logarithmic terms fetched from Theorem \ref{thm:fullonedolo}.

First, we proof the following result that was used in needed in the proofs of Theorem \ref{thm:fullmatrix}:
\begin{Lemma}\label{thm:addnorm}
Suppose $\|\cdot\|_1$ and $\|\cdot\|_2$ are such that $\frac{1}{2}\|x\|_i^2$ is $\sigma_i$-strongly convex with respect to $\|\cdot\|_i$ for $i\in\{1,2\}$. Then the $\|x\|=\sqrt{\|x\|_1^2+\|x\|_2^2}$ is a seminorm and is $\min(\sigma_1,\sigma_2)$-strongly convex with respect to $\|\cdot\|$.
\end{Lemma}
\begin{proof}
First, we show that $\|\cdot\|$ is a seminorm. It is clear that $\|0\|=0$ and $c\|x\|=\|cx\|$. To check triangle inequality, we have
\begin{align*}
    \|x+y\|&=\sqrt{\|x+y\|_1^2 + \|x+y\|_2^2}\\
    &\le \sqrt{(\|x\|_1+\|y\|_1)^2+(\|x\|_2+\|y\|_2)^2}\\
    &=\|(\|x\|_1, \|x\|_2) + (\|y\|_1, \|y\|_2)\|_{\mathbf{2}}\\
    &\le \|(\|x\|_1, \|x\|_2)\|_{\mathbf{2}} + \|(\|y\|_1,\|y\|_2)\|_{\mathbf{2}}\\
    &=\|x\|+\|y\|
\end{align*}

Now we show the strong-convexity. Recall that a function $f$ is $\sigma$-strongly convex if and only if for all $p\in[0,1]$ and all $x,y$,
\begin{align*}
    f\left(px + (1-p)y\right)\le pf(x) + (1-p)f(y) - \frac{\sigma p(1-p)}{2}\|x-y\|^2
\end{align*}
Let $\sigma = \min(\sigma_1, \sigma_2)$. Then we have
\begin{align*}
    \frac{1}{2}\|px + (1-p)y\|_1^2 &\le \frac{p}{2}\|x\|_1^2+\frac{1-p}{2}\|y\|_1^2 + \frac{\sigma p (1-p)}{2} \|x-y\|_1^2\\
     \frac{1}{2}\|px + (1-p)y\|_2^2 &\le \frac{p}{2}\|x\|_2^2+\frac{1-p}{2}\|y\|_2^2 + \frac{\sigma p (1-p)}{2} \|x-y\|_2^2
\end{align*}
Adding these two inequalities proves the stated strong-convexity.
\end{proof}

\begin{restatable}{Theorem}{thmfullmatrixdetail}\label{thm:fullmatrixdetail}
Suppose $g_t$ satisfies $\|g_t\|\le 1$ for all $t$ where $\|\cdot\|$ is a norm such that $\frac{1}{2}\|\cdot\|^2$ is $\sigma$-strongly convex with respect to $\|\cdot\|$. Let $G_t = \sum_{i=1}^t g_ig_i^\top$ and let $r$ be the rank of $G_T$. Suppose we run Algorithm \ref{alg:varynorm} with $\|x\|_t^2=\|x\|^2+x^\top(I + G_t)x$, where $I$ is the identity matrix. Then we obtain regret:
\begin{align*}
    R_T(\w)&\le  \epsilon + 2\|\w\|_T\max\left[\sqrt{\left(3+3r\log(T+1)\right)\log\left(e+\frac{\|\w\|_{T}(7 + 4r\log(T+1))}{\epsilon}\right)}, \right.\\
    &\quad\quad\quad\quad\quad\left. 2 \log\left(e+\frac{\|\w\|_{T}(7 + 4r\log(T+1))}{\epsilon}\right)\right] + \frac{2}{\sqrt{\min(\sigma,1)}}\|\w\|_T\sqrt{1+r\log(T+1)}
\end{align*}
\end{restatable}

\begin{proof}
We saw in the proof of Theorem \ref{thm:fullmatrix} that $\|\w\|_{T-1}\le \|w\|_T=\sqrt{2\|w\|_{\mathbf{2}}^2 +\sum_{t=1}^T \langle g_t, \w\rangle^2}$. We also saw:
\begin{align*}
    \sum_{t=1}^T \|g_t\|_{t-1,\star}^2 &\le \text{rank}(G_T)\log(T+1)
\end{align*}

So then with all constants, the regret is
\begin{align*}
    R_T(\w)&\le \epsilon+ 2\|\w\|_{T-1}\max\left[\sqrt{\left(3+3\sum_{t=1}^T \|g_t\|_{t-1,\star}^2\right)\log\left(e+\frac{\|\w\|_{T-1}(7 + 4\sum_{t=1}^T \|g_t\|_{t-1,\star}^2)}{\epsilon}\right)}, \right.\\
    &\quad\quad\quad\quad\quad\left. 2 \log\left(e+\frac{\|\w\|_{T}(7 + 4\sum_{t=1}^T \|g_t\|_{t-1,\star}^2)}{\epsilon}\right)\right] \\
    &\quad\quad+\frac{2}{\sqrt{\min(\sigma,1)}}\sqrt{1+\sum_{t=1}^{T-1} \|g_t\|_{t-1,\star}^2}\\
    &\le \epsilon + 2\|\w\|_T\max\left[\sqrt{\left(3+3r\log(T+1)\right)\log\left(e+\frac{\|\w\|_{T}(7 + 4r\log(T+1))}{\epsilon}\right)}, \right.\\
    &\quad\quad\quad\quad\quad\left. 2 \log\left(e+\frac{\|\w\|_{T}(7 + 4r\log(T+1))}{\epsilon}\right)\right] \\
    &\quad\quad+\frac{2}{\sqrt{\min(\sigma,1)}}\|\w\|_T\sqrt{1+r\log(T+1)}
\end{align*}

\end{proof}

Next, we carry out a similar computation for the AdaGrad-style full-matrix algorithm:
\begin{restatable}{Theorem}{thmadagraddetail}\label{thm:adagraddetail}
Suppose $W\subset \R^d$ and $g_t$ satisfies $\|g_t\|_{\mathbf{2}}\le 1$ for all $t$. Let $G_t= \sum_{i=1}^t g_ig_i^\top$. Define $\|\cdot\|_t$ be $\|x\|_t^2=x^\top (I+G_t)^{1/2} x$. Then the regret of Algorithm \ref{alg:varynorm} using these norms is bounded by:
\begin{align*}
    R_T(\w)&\le \tilde \epsilon + 2\|\w\|_T\max\left[\sqrt{\left(3+6\trace(G_T^{1/2})\right)\log\left(e+\frac{\|\w\|_{T}(7 + 8\trace(G_T^{1/2}))}{\epsilon}\right)}, \right.\\
    &\quad\quad\quad\quad\quad\left. 2 \log\left(e+\frac{\|\w\|_{T}(7 + 8\trace(G_T^{1/2}))}{\epsilon}\right)\right] +2\|\w\|_T\sqrt{1+2\trace(G_T^{1/2})}
\end{align*}
where the $\tilde O$ notation hides a logarithmic dependency on $\trace\left(G_T^{1/2}\right)\sqrt{\|\w\|_{\textbf{2}}^2 + \w^{\top}G_T^{1/2}\w}$.
\end{restatable}
This Theorem recovers the desired bound (\ref{eqn:adagradbound}) up to log factors. Moreover, it is possible to interpret the operation of the algorithm as in some rough sense ``learning the optimal learning rate'' required for the original AdaGrad algorithm to achieve this bound.
\begin{proof}
In the proof of Theorem \ref{thm:adagrad}, we saw $\|\w\|_{T-1}\le \|\w\|_T = \sqrt{\|\w\|_{\mathbf{2}}^2 + \w^\top G_T^{1/2} \w}$. Further,
\begin{align*}
    \sum_{t=1}^T \|g_t\|_{t-1,\star}^2&\le 2\trace(G_T^{1/2})
\end{align*}

So then with all constants, the regret is
\begin{align*}
    R_T(\w)&\le \epsilon+ 2\|\w\|_{T-1}\max\left[\sqrt{\left(3+3\sum_{t=1}^T \|g_t\|_{t-1,\star}^2\right)\log\left(e+\frac{\|\w\|_{T-1}(7 + 4\sum_{t=1}^T \|g_t\|_{t-1,\star}^2)}{\epsilon}\right)}, \right.\\
    &\quad\quad\quad\quad\quad\left. 2 \log\left(e+\frac{\|\w\|_{T}(7 + 4\sum_{t=1}^T \|g_t\|_{t-1,\star}^2)}{\epsilon}\right)\right] \\
    &\quad\quad+\frac{2}{\sqrt{\sigma}}\sqrt{1+\sum_{t=1}^{T-1} \|g_t\|_{t-1,\star}^2}\\
    &\le \epsilon + 2\|\w\|_T\max\left[\sqrt{\left(3+6\trace(G_T^{1/2})\right)\log\left(e+\frac{\|\w\|_{T}(7 + 8\trace(G_T^{1/2}))}{\epsilon}\right)}, \right.\\
    &\quad\quad\quad\quad\quad\left. 2 \log\left(e+\frac{\|\w\|_{T}(7 + 8\trace(G_T^{1/2}))}{\epsilon}\right)\right] +2\|\w\|_T\sqrt{1+2\trace(G_T^{1/2})}
\end{align*}
\end{proof}

\section{Full Version of Theorem \ref{thm:1dolo} with Constants}\label{sec:1doloproof}
In this section, we provide a more detailed version of Theorem \ref{thm:1dolo} including all logarithmic and constant factors. The proof is essentially a (slightly looser) version of analysis in \cite{cutkosky2019matrix}, but we provide it below for completeness.

\begin{restatable}{Theorem}{thmfullonedolo}\label{thm:fullonedolo}
There exists a one-dimensional online linear optimization algorithm such that if $|g_t|\le 1$ for all $t$, the regret is bounded by
\begin{align*}
    \sum_{t=1}^T g_t(w_t-\w)&\le  \epsilon+ 2|\w|\max\left[\sqrt{\left(3+3\sum_{t=1}^T g_t^2\right)\log\left(e+\frac{|\w|(7 + 4\sum_{t=1}^T g_t^2)}{\epsilon}\right)}, \right.\\
    &\quad\quad\quad\quad\left. 2 \log\left(e+\frac{|\w|(7 + 4\sum_{t=1}^T g_t^2)}{\epsilon}\right)\right] 
\end{align*}
And moreover each $w_t$ is computed in $O(1)$ time.
\end{restatable}

\begin{proof}

Define the \emph{wealth} of an algorithm as:
\begin{align*}
\wealth_t = \epsilon - \sum_{\tau=1}^tg_\tau w_\tau
\end{align*}
We set
\begin{align*}
    w_{t+1} = v_{t+1} \wealth_t
\end{align*}
where $v_t\in[-1/2,1/2]$. This implies:
\begin{align*}
    \wealth_T = \epsilon \prod_{t=1}^{T}(1-g_t v_t) 
\end{align*}
Define 
\begin{align*}
    \wealth_T(\vv) = \epsilon\prod_{t=1}^T (1-g_t \vv)
\end{align*}
Now, to choose $v_t$, consider the functions:
\begin{align*}
    \ell_t(v) = -\log(1-g_t v)
\end{align*}
Observe that $\ell_t(v)$ is convex. Let $z_t=\frac{g_t}{1-g_tv_t}=\ell_t'(v_t)$. Notice that $|z_t|\le 2|g_t|\le 2$ since $v_t\in[-1/2,1/2]$. Then we have
\begin{align*}
    \log\left(\wealth_T(\vv)\right) - \log\left(\wealth_T\right) = \sum_{t=1}^T \ell_t(v_t) -\ell_t(\vv)\le \sum_{t=1}^T z_t(v_t - \vv)
\end{align*}
Now we choose $v_t\in[-1/2,1/2]$ using FTRL on the losses $z_t$ with regularizers
\begin{align*}
    \psi_t(v) = \frac{Z}{2}(5+\sum_{\tau=1}^tz_\tau^2) v^2
\end{align*}
Notice that $\psi_t$ is $Z(4+\sum_{\tau=1}^t z_\tau^2)$-strongly convex with respect to $|\cdot|$. Therefore by Theorem \ref{thm:ftrl}:
\begin{align*}
     \sum_{t=1}^T z_t(v_t - \vv)&\le \psi_T(\vv) + \frac{1}{2}\sum_{t=1}^T \frac{z_t^2}{Z(5+\sum_{\tau=1}^{t-1} z_\tau^2)}\\
     &\le \frac{Z}{2}\left(5+\sum_{t=1}^T z_t^2\right)\vv^2 + \frac{1}{2Z}\sum_{t=1}^T \frac{z_t^2}{1+\sum_{\tau=1}^{t} z_\tau^2}\\
     &\le \frac{Z}{2}\left(5+\sum_{t=1}^T z_t^2\right)\vv^2 + \frac{1}{2Z}\log\left(1+ \sum_{t=1}^T z_t^2\right)
\end{align*}

Therefore, for all $\vv\in[-1,2/,1/2]$,
\begin{align*}
    \log\left(\wealth_T\right)&\ge \log\left(\wealth_T(\vv)\right) - \frac{Z}{2}\left(5+\sum_{t=1}^T z_t^2\right)\vv^2 + \frac{1}{2Z}\log\left(1+ \sum_{t=1}^T z_t^2\right)\\
    &\ge \log\left(\wealth_T(\vv)\right) - \frac{Z}{2}\left(5+4\sum_{t=1}^T g_t^2\right)\vv^2 + \frac{1}{2Z}\log\left(1+ 4\sum_{t=1}^T g_t^2\right)
\end{align*}
Next, use the tangent bound $\log(1-x)\ge -x-x^2$ to obtain:
\begin{align*}
    \log\left(\wealth_T(\vv)\right)&\ge \log(\epsilon)-\sum_{t=1}^T g_t \vv - \sum_{t=1}^T g_t^2\vv^2
\end{align*}
So overall we have:
\begin{align*}
    \log\left(\wealth_T\right)&\ge \log(\epsilon)-\sum_{t=1}^T g_t \vv - \frac{Z}{2}\left(5+\sum_{t=1}^T \left(\frac{2}{Z} + 4\right)g_t^2\right)\vv^2-\frac{1}{2Z}\log\left(1+ 4\sum_{t=1}^T g_t^2\right)\\
    \wealth_T&\ge \epsilon\exp\left(-\sum_{t=1}^T g_t \vv - \frac{Z}{2}\left(5+\sum_{t=1}^T (\frac{2}{Z} + 4)g_t^2\right)\vv^2-\frac{1}{2Z}\log\left(1+ 4\sum_{t=1}^T g_t^2\right)\right)
\end{align*}
Now we relate this to regret:
\begin{align*}
    \sum_{t=1}^T g_t(w_t - \w)& = \epsilon - \w\sum_{t=1}^T g_t - \wealth_T\\
    &\le \epsilon - \w\sum_{t=1}^T g_t - \epsilon\exp\left(-\sum_{t=1}^T g_t \vv - \frac{Z}{2}\left(5+\sum_{t=1}^T \left(\frac{2}{Z} + 4\right)g_t^2\right)\vv^2-\frac{1}{2Z}\log\left(1+ 4\sum_{t=1}^T g_t^2\right)\right)\\
    &\le \epsilon+ \sup_{G} \left[G\w - \epsilon\exp\left(G \vv - \frac{Z}{2}\left(5+\sum_{t=1}^T \left(\frac{2}{Z} + 4\right)g_t^2\right)\vv^2-\frac{1}{2Z}\log\left(1+ 4\sum_{t=1}^T g_t^2\right)\right)\right]\\
    &\le \epsilon +\frac{|\w|}{\vv}\left(\log\left(\frac{|\w|}{\epsilon \vv}\right)+\frac{Z}{2}\left(5+\sum_{t=1}^T \left(\frac{2}{Z} + 4\right)g_t^2\right)\vv^2+\frac{1}{2Z}\log\left(1+ 4\sum_{t=1}^T g_t^2\right)-1\right)\\
    &\le \epsilon +\frac{|\w|}{\vv}\log\left(\frac{|\w|(1 + 4\sum_{t=1}^T g_t^2)^{1/2Z}}{\epsilon \vv}\right) + \frac{Z}{2}\left(5 + \sum_{t=1}^T \left(\frac{2}{Z} + 4\right)g_t^2\right)\vv
\end{align*}
where we have used \cite{cutkosky2019matrix} Lemma 3 in to calculate the supremum over $G$.
Now set $Z=1$, apply \cite{cutkosky2019matrix} Lemma 4, and over-approximate several constants to obtain:
\begin{align*}
    \sum_{t=1}^T g_t(w_t- \w)&\le  \epsilon+ 2|\w|\max\left[\sqrt{\left(3+3\sum_{t=1}^T g_t^2\right)\log\left(e+\frac{|\w|(7 + 4\sum_{t=1}^T g_t^2)}{\epsilon}\right)}, \right.\\
    &\quad\quad\quad\quad\left. 2 \log\left(e+\frac{|\w|(7 + 4\sum_{t=1}^T g_t^2)}{\epsilon}\right)\right] 
\end{align*}

\end{proof}

% \begin{Proposition}
% For any $a,b,c$, $\sup_{X} Xy-a\exp(bX + c)= \frac{|y|}{b}\left(\log\left(\frac{|y|}{ab}\right) - c -1\right)$
% \end{Proposition}
% \begin{Proposition}
% For all $|x|\le 1/2$, $\log(1-x)\ge -x + 4(1/2+\log(1/2))x^2\ge -x-x^2$
% \end{Proposition}
% \begin{proof}
% By Taylor expansion, we have
% \begin{align*}
%     \log(1-x) &= -x - \sum_{i=2}^\infty \frac{x^i}{i}\\
%     &\ge -x - \sum_{i=2}^\infty \frac{x^2}{i2^{i-2}}\\
%     &= -x -4x^2\left(-\frac{1}{2}+ \sum_{i=1}^\infty \frac{1}{i2^i}\right)\\
%     &= -x -4x^2\left(-\frac{1}{2} -\log(1-1/2)\right)
% \end{align*}
% Where we have used the Taylor expansion at $x=1/2$ in the last line.
% \end{proof}

\section{Missing Proofs for Scale-Invariance}\label{sec:scaleproofs}

Now we provide the proof of Theorem \ref{thm:diagscale}, restated below:
\thmdiagscale*

\begin{proof}
Observe that Algorithm \ref{alg:diagscale} is running an independent learner on each coordinate.
Since we can decompose the regret as
\begin{align*}
    \sum_{t=1}^T \langle g_t, w_t-\w\rangle =\sum_{i=1}^d\sum_{t=1}^T g_{t,i}(w_{t,i}-\w_i)
\end{align*}
it suffices to bound the regret for one dimension and then sum over dimensions to get the final regret bound. To this end, we will consider only one dimension and drop all the subscript $i$s.

We have $\|g_t\|_{t-1,\star}=|g_t|/m_t$ also, so $\|g_t\|_{t-1,\star}\le |\nabla_t|\frac{f_t}{m_t}\le 1$ for all $t$. Therefore by Theorem \ref{thm:varynorm}, the regret for one coordinate is
\begin{align*}
    \tilde O\left[\epsilon+|\w|\sqrt{M_T^2\sum_{t=1}^T \nabla_t^2 \frac{f_t^2}{m_t^2}\log\left(1+\frac{\sum_{t=1}^T \nabla_{t}^2\frac{f_{t}^2}{m_{t}^2}|\w_i|}{\epsilon}\right)}\right]
\end{align*}
so that summing over all coordinates proves the given regret bound.

To see that the algorithm is scale-invariant, we need to appeal to the internals of Algorithm \ref{alg:varynorm}. To start, observe that again it suffices to prove scale-invariance in the one-dimensional setting as the independent updates for each coordinate will then imply scale-invariance with respect to diagonal transformations. Next, notice that in the unconstrained setting, Algorithm \ref{alg:varynorm} is the same as Algorithm \ref{alg:1dred}, and so we need not concern ourselves with the effects of projection operators. Now consider two sequences of features $f_1,\dots,f_T$ and $Mf_1,\dots,Mf_T$ for some scalar $M\ne 0$. For any relevant variable $z$ we will use $z_t$ to indicate the $t$th value of that variable when running an algorithm using $f_1,\dots,f_T$, and the $z_{t,M}$ to indicate the $t$th value of that variable when running an algorithm using $Mf_1,\dots,Mf_T$, so that for example $f_{t,M}=Mf_t$. Note that since $w_1=w_{1,M}=0$, we have $f_1w_1 =f_{1,M} w_{1,M}$. Further, we have $x_{1,M}=0= x_1/M$ and $y_1=y_{1,M}$, where $x_i$ and $y_i$ indicate the outputs of FTRL and the one-dimensional parameter-free subroutines in Algorithm \ref{alg:1dred}. Suppose for purposes of induction that $y_{t'}=y_{t',M}$ and $f_{t'} x_{t'}= f_{t',M} x_{t',M}$ for all $t'\le t$. Note that this implies that $f_{t'} w_{t'} =  f_{t',M}w_{t',M}$ for all $t'\le t$. Then we must have $\nabla_{t'}=\nabla_{t',M}$ for all $t'\le t$ and so $g_{t',M}=Mg_{t'}$ for all $t'\le t$. Finally, we also have $m_{t',M} = Mm_{t'}$ for all $t'\le t$. From this we can conclude that $\|x\|_{t,M}=|x|m_{t+1, M}=\|Mx\|_t$ for arbitrary invertible $M$. Further, we have 
\begin{align*}
    \|g_{t,M}\|_{t,M,\star} &= |Mg_t|/m_{t+1,M}\\
    &=|g_t|/m_{t+1}\\
    &=\|g_t\|_{t,\star}
\end{align*}
So that the regularizers used in the FTRL subroutine of Algorithm \ref{alg:varynorm} satisfy $\psi_t(Mx)=\psi_{t,M}(x)$. Since $\sum_{i=1}^t g_{i,M}=M\sum_{i=1}^t g_i$, this implies that the output of the FTRL subroutine, $x_{t+1}$, satisfies $x_{t+1} = M x_{t+1,M}$ so that $f_{t+1} x_{t+1} = f_{t+1, M}, x_{t+1,M}$. Further, note that $s_{t'} =  g_{t'}x_{t'}$, so that by the induction hypothesis, $s_{t',M}=s_{t'}$ for $t'\le t$. Since $y_{t+1}$ depends only on $s_{t'}$ for $t'\le t$, we have $y_{t+1}=y_{t+1,M}$ and so by induction the algorithm is scale-invariant for all time steps.

\end{proof}

\subsection{Full-Matrix Scale-Invariance}\label{sec:scalefull}
In this section we provide an algorithm that achieves scale-invariance with respect to any invertible matrix. That is, we now allow each $f_t$ to be replaced by $Mf_t$ for an arbitrary invertible matrix $M$, while still asking that the predictions $\langle f_t, w_t\rangle$ remain unchanged. Our analysis technique essentially combines the method of Theorem \ref{thm:diagscale} with that of Theorem \ref{thm:fullmatrix}. Note that in this case our regret bound will be \emph{linear} in the rank $r$ of $\sum_{t=1}^T g_tg_t^\top$, which is worse than the bound (\ref{eqn:fullmatrix}), but matches best-known scale-invariant algorithms \cite{kotlowski2019scale, luo2016efficient}. Moreover, we can use this result to easily match exactly the diagonal scale-invariant bounds of the second algorithm in \cite{kempka2019adaptive}: simply run a one-dimensional copy of Algorithm \ref{alg:scale} in each coordinate.
\begin{algorithm}
   \caption{Full-Matrix Scale-Invariance}
   \label{alg:scale}
\begin{algorithmic}
   \STATE{\bfseries Input: } Vector Space $W$.
   \STATE Initialize Algorithm \ref{alg:varynorm}.
   \STATE Receive $f_1$
   \STATE Set $\|x\|_0^2= \langle f_1,x\rangle^2$.
   \FOR{$t=1\dots T$}
   \STATE Get $t$th output $w_t\in W$ from Algorithm \ref{alg:varynorm}.
   \STATE Output $w_t$.
   \STATE Get loss $\ell_t(\cdot)=c_t(\langle f_t, \cdot\rangle)$.
   \STATE Set $\nabla_t \in \partial c_t(\langle f_t, w_t\rangle)$.
   \STATE Set $g_t=\nabla _t f_t\in\partial \ell_t(w_t)$.
   \STATE Get feature vector $f_{t+1}$.
   \STATE Set $G_t = \sum_{i=1}^t g_t g_t^\top$.
   \STATE Define $\|x\|_t = \sqrt{\|x\|_{G_t}^2+2\max_{i\le t+1} \langle f_i, x\rangle^2}$
   \STATE Send $g_t$ and $\|\cdot\|_t$ to Algorithm \ref{alg:varynorm} as $t$th loss and norm respectively,
   \ENDFOR
\end{algorithmic}
\end{algorithm}

\begin{restatable}{Theorem}{thmscaleinvariant}\label{thm:scaleinvariant}
Suppose $|\nabla_t|\le 1$ is 1-Lipschitz for all $t$. Then
Algorithm \ref{alg:scale} is scale-invariant with respect to any invertible linear transformation and achieves regret:
\begin{align*}
    R_T(\w)&\le O\left(r\sqrt{\left(\max_{t} \langle f_t, \w\rangle^2 +\sum_{t=1}^T \langle g_t, \w\rangle^2\right)}\log\left(\frac{2\sum_{t=1}^T \|g_t\|_{\mathbf{2}}^2}{r(r+1)\lambda_\star}\right)\right)
\end{align*}
where again $r$ is the rank of $G_T$ and $\lambda_\star$ is the minimum non-zero eigenvalue of any $G_t$.
\end{restatable}
Note that if we were to run a one-dimensional copy of this algorithm on each coordinate of the problem, we would obtain an algorithm that is invariant to diagonal transformations with a regret bound matching that of \cite{kempka2019adaptive} in both logarithmic terms and dependence on $g_t$.
\begin{proof}
Our first task is to show that $\|\cdot\|_t$ is a valid seminorm. To do this we show first that the maximum of any two seminorms is a seminorm, which implies that $\sqrt{\max_{i\le t+1}\langle f_i, x\rangle^2}$ is a seminorm. Combined with Lemma \ref{thm:addnorm}, this shows that $\|\cdot\|_t$ is a seminorm for all $t$. To see that the maximum of two seminorms $\|\cdot\|=\max(\|\cdot\|_a, \|\cdot\|_b)$ is a seminorm, observe that clearly the maximum satisfies $a\|x\|=\|ax\|$ so that we need only check the triangle inequality. For this we have
\begin{align*}
    \|x+y\|&=\max(\|x+y\|_a,\|x+y\|_b)\\
    &\le \max(\|x\|_a+\|y\|_a, \|x\|_b+\|y\|_b)\\
    &\le \max(\|x\|_a,\|x\|_b)+\max(\|y\|_a,\|y\|_b)\\
    &=\|x\|+\|y\|
\end{align*}

Next, since $c_t$ is 1-Lipschitz, we must have $|\nabla_t|\le 1$. Finally, we have
\begin{align*}
    \|g_t\|_{t-1,\star}&=\sup_{\|x\|_{t-1}\le 1} \langle g_t, x\rangle\\
    &\le \sup_{\langle f_1, x\rangle^2+langle f_t,x\rangle^2 + \sum_{i=1}^{t-1} \langle g_t, x\rangle^2\le 1}\langle g_t, x\rangle\\
    &\le \sup_{\langle f_1, x\rangle^2+\langle \nabla_tf_t,x\rangle^2 + \sum_{i=1}^{t-1} \langle g_t, x\rangle^2\le 1}\langle g_t, x\rangle\\
    &=\sup_{\langle f_1, x\rangle^2+\sum_{i=1}^{t} \langle g_t, x\rangle^2\le 1}\langle g_t, x\rangle\\
    &=\|g_t\|_{(G_t+f_1f_1^\top)^{-1}}\\
    &\le 1
\end{align*}

Now from direct application of Theorem \ref{thm:varynorm}, we have
\begin{align*}
    R_T(\w)&\le \tilde O\left(\|\w\|_{T-1}\sqrt{\sum_{t=1}^T \|g_t\|_{t-1,\star}^2}\right)\\
    &=\tilde O\left(\sqrt{\max_t\langle f_t, \w\rangle^2 + \sum_{t=1}^T \langle g_t, \w\rangle^2}\sqrt{\sum_{t=1}^T \|g_t\|_{G_t^{-1}}}\right)
\end{align*}
Now we apply Theorem 4 of \cite{luo2016efficient}, which states:
\begin{align*}
\sum_{t=1}^T \|g_t\|_{G_t^{-1}}&\le r + \frac{r(r+1)}{2}\log\left(1+\frac{2\sum_{t=1}^T \|g_t\|_{\mathbf{2}}^2}{r(r+1)\lambda_\star}\right)
\end{align*}
And so the regret bound follows.

To see that the algorithm is scale-invariant, we need to examine the update of Algorithm \ref{alg:varynorm} in a little more detail. To start, observe that since we consider $W$ to be an entire vector space, Algorithm \ref{alg:varynorm} is in fact identical to Algorithm \ref{alg:1dred}, so we may restrict our attention to that algorithm instead. Consider two sequences of features $f_1,\dots,f_T$ and $Mf_1,\dots,Mf_T$ for some invertible matrix $M$. For any relevant variable $z$ we will use $z_t$ to indicate the $t$th value of that variable when running an algorithm using $f_1,\dots,f_T$, and the $z_{t,M}$ to indicate the $t$th value of that variable when running an algorithm using $Mf_1,\dots,Mf_T$, so that for example $f_{t,M}=Mf_t$. Note that since $w_1=w_{1,M}=0$, we have $\langle f_1, w_1\rangle = \langle f_{1,M}, w_{1,M}\rangle$. Further, we have $x_{1,M}=0=(M^{-1})^\top x_1$ and $y_1=y_{1,M}$, where $x_i$ and $y_i$ indicate the outputs of FTRL and the one-dimensional parameter-free subroutines in Algorithm \ref{alg:1dred}. Suppose for purposes of induction that $y_i=y_{i,M}$ and $\langle f_i, x_i\rangle= \langle f_{i,M}, x_{i,M}\rangle$ for all $i\le t$. Note that this implies that $\langle f_i, w_i\rangle = \langle f_{i,M},w_{i,M}\rangle$ for all $i\le M$. Then we must have $\nabla_i=\nabla_{i,M}$ for all $i\le t$ and so $g_{i,M}=Mg_i$ for all $i\le t$. From this we can conclude that $\|x\|^2_{t,M}=x^\top MG_TM\top x=\|M^\top x\|_t$ for arbitrary invertible $M$. Further, we have 
\begin{align*}
    \|g_{t,M}\|_{t,M,\star} &= \sup_{\|x\|_{t,M}\le 1}\langle Mg_t, x\rangle\\
    &=\sup_{\|x\|_t\le 1}\langle Mg_t,(M^{-1})^\top x\rangle\\
    &=\sup_{\|x\|_t\le 1}\langle g,x\rangle\\
    &=\|g_t\|_{t,\star}
\end{align*}
So that the regularizers used in the FTRL subroutine of Algorithm \ref{alg:varynorm} satisfy $\psi_t(M^\top x)=\psi_{t,M}(x)$. Since $\sum_{i=1}^t g_{i,M}=M\sum_{i=1}^t g_i$, this implies that the output of the FTRL subroutine, $x_{t+1}$, satisfies $(M^{-1})^\top x_{t+1} = x_{t+1,M}$ so that $\langle f_{t+1}, x_{t+1}\rangle = \langle f_{t+1, M}, x_{t+1,M}\rangle$. Further, note that $s_i = \langle g_i, x_i\rangle$, so that by the induction hypothesis, $s_{i,M}=s_i$ for $i\le t$. Since $y_{t+1}$ depends only on $s_i$ for $i\le t$, we have $y_{t+1}=y_{t+1,M}$ and so by induction the algorithm is scale-invariant for all time steps.
\end{proof}

% \begin{Lemma}\label{thm:improvedbound}
% Let $g_1,\dots,g_T$ be a sequence of vectors with $\|g_t\|_{\mathbf{2}}\le G$. Let $r$ be the rank of the subspace spanned by $g_1,\dots,g_T$. Let $V_t = xx^\top + \sum_{i=1}^t g_ig_i^\top$ for some fixed vector $x$. Let $V_t^{\dagger}$ be the psuedo-inverse of $V_t$. Let $|V_t|$ be the product of the non-zero eigenvalues of $V_t$. Then
% \begin{align*}
%     \sum_{t=1}^Tg_t^\top V_t^{\dagger}g_t\le r + 2\log(G^2T/\|x\|_{\mathbf{2}+1))
% \end{align*}
% \end{Lemma}
% \begin{proof}
% Our proof takes much inspiration from the similar Lemma 11 and related Lemma 12 of \cite{hazan2007logarithmic}. To start with, we write
% \begin{align*}
%     \sum_{t=1}^T g_t^\top V_t^{\dagger} g_t&=\sum_{t=1}^T \trace\left(V_t^{\dagger}g_tg_t^\top\right)\\
%     &=\sum_{t=1}^T \trace\left(V_t^{\dagger}(V_t- V_{t-1}\right)\\
%     &=\sum_{t=1}^T r_t-\trace\left( V_t^{\dagger}V_{t-1}\right)
% \end{align*}
% Where $r_t$ is the rank of $V_t$. Notice that $r_t$ can increase by at most 1 at every $t$, and $r_t$ can increase at most $r$ times. Further, let

% \end{proof}

% \section{My Proof of Theorem 1}

% This is a boring technical proof.

% \section{My Proof of Theorem 2}

% This is a complete version of a proof sketched in the main text.

\end{document}